\documentclass[11pt]{article}
\usepackage[includeheadfoot,margin=1in]{geometry}
\usepackage{times}
\usepackage{amsmath,amsfonts,amssymb,amsthm,bm}
\usepackage{mathrsfs}
\usepackage{mathtools}
\usepackage{enumerate}
\usepackage{verbatim}
\usepackage{commath}
    \usepackage[usenames]{color}
    \definecolor{plum}  {rgb}{.4,0,.4}
    \definecolor{BrickRed} {rgb}{0.6,0,0}
	\definecolor{DarkBlue} {rgb}{0,0,0.6}
\usepackage[plainpages=false,pdfpagelabels,colorlinks=true,linkcolor=BrickRed,citecolor=plum,urlcolor=DarkBlue]{hyperref}
\usepackage[mathcal]{euscript}
\usepackage[round]{natbib}
\usepackage{cleveref}

\def\ddefloop#1{\ifx\ddefloop#1\else\ddef{#1}\expandafter\ddefloop\fi}

\def\ddef#1{\expandafter\def\csname b#1\endcsname{\ensuremath{\boldsymbol{#1}}}}
\ddefloop ABCDEFGHIJKLMNOPQRSTUVWXYZabcdefghijklnopqrstuvwxyz\ddefloop

\def\ddef#1{\expandafter\def\csname c#1\endcsname{\ensuremath{\mathcal{#1}}}}
\ddefloop ABCDEFGHIJKLMNOPQRSTUVWXYZ\ddefloop
 
\def\ddef#1{\expandafter\def\csname s#1\endcsname{\ensuremath{\mathsf{#1}}}}
\ddefloop ABCDEFGHIJKLMNOPQRSTUVWXYZ\ddefloop

\def\Reals{{\mathbb R}}

\def\Ex{{\mathbf E}} 
\def\Pr{{\mathbf P}} 
\def\eps{\varepsilon}

\def\trn{{\mathsf T}} 

\newcommand{\Sin}{S_{\mathsf{in}}}
\newcommand{\Sout}{S_{\mathsf{out}}}

\newcommand{\ii}{{\mathsf{ii}}}
\newcommand{\io}{{\mathsf{io}}}
\newcommand{\oi}{{\mathsf{oi}}}
\newcommand{\oo}{{\mathsf{oo}}}

\makeatletter
\newsavebox{\@brx}
\newcommand{\llangle}[1][]{\savebox{\@brx}{\(\m@th{#1\langle}\)}%
  \mathopen{\copy\@brx\kern-0.5\wd\@brx\usebox{\@brx}}}
\newcommand{\rrangle}[1][]{\savebox{\@brx}{\(\m@th{#1\rangle}\)}%
  \mathclose{\copy\@brx\kern-0.5\wd\@brx\usebox{\@brx}}}
\makeatother

\newtheorem{theorem}{Theorem}

\newtheorem{lemma}{Lemma}
\newtheorem{corollary}{Corollary}
\newtheorem{remark}{Remark}

\begin{document}

\title{Separating Geometry from Probability\\
in the Analysis of Generalization}

\author{Maxim Raginsky\thanks{University of Illinois, Urbana, IL 61801, USA}\\ \href{mailto:maxim@illinois.edu}{maxim@illinois.edu} \and Benjamin Recht\thanks{University of California, Berkeley, CA 94720, USA}\\\href{mailto:brecht@berkeley.edu}{brecht@berkeley.edu}}
\date{}
\maketitle

\begin{abstract} The goal of machine learning is to find models that minimize prediction error on data that has not yet been seen. Its operational paradigm assumes access to a dataset $S$ and articulates a scheme for evaluating how well a given model performs on an arbitrary sample. The sample can be $S$ (in which case we speak of ``in-sample'' performance) or some entirely new $S'$ (in which case we speak of ``out-of-sample'' performance). Traditional analysis of generalization assumes that both in- and out-of-sample data are i.i.d.\ draws from an infinite population. However, these probabilistic assumptions cannot be verified even in principle. This paper presents an alternative view of generalization through the lens of sensitivity analysis of solutions of optimization problems to perturbations in the problem data. Under this framework, generalization bounds are obtained by purely deterministic means and take the form of variational principles that relate in-sample and out-of-sample evaluations through an error term that quantifies how close out-of-sample data are to in-sample data. Statistical assumptions can then be used \textit{ex post} to characterize the situations when this error term is small (either on average or with high probability). \end{abstract}

\section{Introduction}
The colloquial explanation for the observed performance of machine learning algorithms is grounded in probabilistic assumptions regarding the data used to train, validate, and ultimately apply the learned models. Data is generated by sampling from a probability distribution. That probability distribution is the same---or at least very close to the same---over time. A model ``generalizes'' if it accurately predicts new samples from this distribution after being trained to predict data previously sampled from this distribution. Even if such probabilistic assumptions and the underlying theoretical frameworks (either frequentist or Bayesian) are not articulated explicitly, they are tacitly invoked whenever terms like ``out-of-distribution" are used as shorthand for scenarios when the learned models are applied in circumstances that had not been adequately covered by the training and the validation datasets. The ``out-of-distribution" designation is called on to do explanatory work both for poor performance of the model and for good performance. In the former case, one talks of ``overfitting to the training data" or of the imperative for scaling; in the latter case, the learned model is inferred to be capable of ``domain adaptation" or ``transfer learning." This terminological drift conflates empirically verifiable claims (e.g., statements regarding sample sizes, specific choices of hyperparameters or random seeds, observed differences between in-sample vs.\ out-of-sample performance or other sample statistics, software implementation choices, etc.) and claims about theoretical entities and constructs whose validity cannot be empirically tested or verified even in principle (e.g., assumptions of independent and identically distributed samples from a fixed but unknown probability distribution). 

In this paper, we present a non-stochastic theory of generalization. More specifically, we show that many common probabilistic arguments about generalization can be derived as expected values of deterministic perturbations of the algorithms used to train machine learning models. We adopt a view of machine learning that has been explored both by \citet{bousquet2002stability} and \citet{Kuchibhotla_2019}: Machine learning can be thought of as \emph{parametric programming}, where the \emph{data} is the perturbation parameter. An algorithm generalizes when the objective is stable under ``reasonable'' perturbations of the data. Unlike prior work applying parametric programming to statistical learning theory, we derive \emph{deterministic} bounds on the error of a model trained on one dataset and evaluated on another. These bounds are often implicitly computed as part of standard probabilistic arguments about machine learning generalization. Here, we show that they can be decoupled from the probability and provide new insights into deterministic properties of machine learning algorithms. [We should, however, point out that certain methods for deriving generalization bounds in the statistical setting do not admit such a straightforward deterministic recasting. These include, for example, PAC-Bayes bounds \citep{Catoni_2007} and uniform convergence bounds that capitalize on symmetry/exchangeability inherent in certain explicitly probabilistic quantities like the Rademacher complexity \citep{Koltchinskii_2011}.

We note that our approach is different from the standard non-stochastic theory of online learning \citep{PLG_book}. Online learning also makes no probabilistic assumptions about how data is generated. Online algorithms sequentially assemble prediction functions by adding one data point at a time. At each time, a new prediction function is chosen in a model class to predict the next sample. The quality of online algorithms is measured in terms of regret. Regret contains two components, the first is the average error over time of the error made using the first $t$ samples to predict the next ($t+1$) sample. The second is the best prediction possible in the same model class assuming all of the data has been seen. Regret is the difference of these two averages.

Regret measures something different from out-of-sample error, as it is a series of predictors at each time step. The \emph{expected value} of regret, however, is often a generalization error, because now we can assume each term in the regret average is simply the error made when $t$ independent samples are used to predict another identically distributed example. This technique, called \emph{online to batch} conversion \citep{Cesa_Bianchi_etal_2004}, is similar to the methods we use to convert our deterministic inequalities about batch learning into probabilistic generalization bounds.

Moreover, while sensitivity analysis also motivating the groundbreaking work on stability by~\citet{bousquet2002stability}, their work immediately applies parametric programming to stochastic optimization. While optimization researchers have used this technique to analyze stochastic problems (see, for example, \citet{shapiro1994quantitative}) they have just as widely, if not more widely, applied the framework to deterministic optimization problems. Interestingly, \citet{shalev2010learnability} show that in many probabilistic models of learning, stability, and hence the sorts of bounds we study in this work, are necessary and sufficient for generalization. We note that modeling assumptions about data generating distributions can be added at the end of the analysis, rather than at the beginning. Indeed, we will discuss throughout how the expected value of our deterministic bounds under such data generating processes match optimal bounds in probabilistic generalization theory.

All of our work proceeds through perturbation analysis of the optimality conditions of optimization problems that arise in machine learning. Adopting the terminology of \citet{Shapiro_1992}, we will refer to such quantitative comparisons as \textit{variational principles}. Each section explores different ways to derive variational principles in machine learning, linking together many lines of prior work. In Section~\ref{sec:interpolation} we revisit variational principles of function interpolation, first derived in the 1950s. We then derive variational principles from duality of convex optimization in Section~\ref{sec:duality}, leading to novel tight bounds on the leave-one-out error of hard margin support vector machines. In Section~\ref{sec:quadratic}, we apply geometric techniques from parametric programming, deriving new generalization bounds and connections to ``basic inequalities'' in mathematical statistics  \citep{VanDeGeer_2000,Paik_2025}. Finally, in Section~\ref{sec:evaluations}, we show that, under convexity assumptions on the model evaluation functional, it is possible to construct purely deterministic versions of localization arguments originally developed by \citet{Hjort_1993} and then refined by \citet{Koltchinskii_2006}. 

\section{The basic setup}

For simplicity, we exclusively focus on a standard supervised learning framework. Here, machine learning systems operate with three basic types of entities: data sets (or samples), models, and evaluations (or benchmarks). Data sets are finite subsets of a Cartesian product $\cZ = \cX \times \cY$, with $\cX$ comprised of covariates to be leveraged to predict elements of $\cY$. Models are functions from $\cX$ to $\cY$. A \textit{model class} is a subset $\cF$ of a Banach space ${\mathfrak B}$ with norm $\|\cdot\|$.

Evaluations are quantitative criteria for assessing the fit or misfit of a given model to a given sample. An example is the average prediction error. Let a sample $S = \{(x_1,y_1),\dots,(x_n,y_n)\}$ and a candidate prediction function $f$ be given, and let $\ell$ be a \emph{loss} function that grows monotonically with misfit. Then the average prediction error is
\begin{align*}
	L(S,f) = \frac{1}{n}\sum^n_{i=1}\ell(y_i,f(x_i))\,.
\end{align*}
An evaluation is thus is a mapping $L$ that takes a sample $S \subset \cZ$ and a model $f \in \cF$ and returns a nonnegative real number $L(S,f)$. We assume that, for each sample $S$, the set
\begin{align*}
\cM(S,0) := \{ f \in \cF: L(S,f) \le L(S,g), \, \forall g \in \cF \}
\end{align*}
of the minimizers of $L(\cdot,S)$ over $\cF$ is nonempty. This implies, in particular, that for each $\eps > 0$ the set
\begin{align*}
\cM(S,\eps) := \{ f \in \cF: L(S,f) \le L(S,g) + \eps, \, \forall g \in \cF \}
\end{align*} 
of the $\eps$-minimizers of $L(\cdot,S)$ over $\cF$ is nonempty as well. A learning algorithm $\cA$ is a rule for assigning models to samples, $f = \cA(S)$. We say that $\cA$ \textit{interpolates} the sample $S$ if $L(S,\cA(S)) = 0$. In machine learning terminology, such an algorithm achieves zero training error on $S$. Another class of algorithms that will be of interest are the ones that have the property $\cA(S) \in \cM(S,\eps)$ for some $\eps \ge 0$. Such an algorithm may or may not interpolate $S$ (e.g., it will be interpolating if $\cA(S) \in \cM(S,0)$ and $\min_{f \in \cF}L(S,f) = 0$).

\paragraph{Generalization.} Statements about generalization are, at their core, about the following question: Given $\Sin$ and $f_{{\mathsf{in}}} = \cA(\Sin)$, what is the best a priori prediction we can make about $L(\Sout,f_{{\mathsf{in}}})$ for an arbitrary $\Sout$? Relatedly, how does $\cA(\Sin)$ perform on $\Sout$ compared to $\cA(\Sout)$?  A \emph{generalization bound} is an explicit prediction of out-of-sample evaluation in terms of in-sample evaluation. That is, if we denote in-sample data and out-of-sample data by $\Sin$ and $\Sout$ explicitly, then a generalization bound is
\begin{align}\label{eq:generalization}
	L(\Sout,f) \leq L(\Sin,f) + {\rm dissim}(\Sout,\Sin).
\end{align}
Here, ${\rm dissim}(\cdot,\cdot)$ is some quantified notion of dissimilarity between the samples.  In order to make any claims about out-of-sample behavior, we must make some inductive assumptions about what the out-of-sample data \emph{is}. Typically in machine learning, the starting assumption is about how in- and out-of-sample data are related. To relate in-sample and out-of-sample evaluations, we will treat $L(\Sout,\cdot)$ as a \textit{perturbation} of $L(\Sin,\cdot)$. In this way, the notion that ``the future is similar to the past" can be made quantitative by positing that this perturbation is small in the sense that ${\rm dissim}(\Sout,\Sin)$ is small.

\paragraph{Variational principles.} As we wrote in the introduction, the problem of minimizing $L(S,f)$ over $f \in \cF$ for a given $S$, either exactly or approximately, is an instance of \textit{parametric programming}, where the role of the parameter is played by the sample $S$. Perturbation analysis of optimization problems \citep{Bonnans_2000} is concerned with the sensitivity of the solutions of parametric programs to perturbations of the parameter. Denoting by $f_S$ an exact or an approximate minimizer of $L(S,f)$, one can obtain either bounds on the norm $\| f_S - f_{S'} \|$ or on the difference of values $L(S,f_S) - L(S',f_{S'})$. In the context of machine learning, $S = \Sin$, $S' = \Sout$, and we seek to compare $\cA(\Sin)$ and $\cA(\Sout)$ for a given learning algorithm $\cA$. 

If we follow the usual sensitivity analysis route, we will need to examine either the norm $\| \cA(\Sin) - \cA(\Sout) \|$ or the difference of values $L(\Sin,\cA(\Sin)) - L(\Sout,\cA(\Sout))$. However, for the purposes of analyzing generalization, we will also need to compare $L(\Sout,\cA(\Sin))$ to $L(\Sin,\cA(\Sin))$. This latter comparison encodes the ``arrow of time'' in the setting of machine learning, where we view the in-sample $\Sin$ as having been available in the past, $\cA(\Sin)$ as the model that was constructed on the basis of past data, the out-sample $\Sout$ as a potential future sample, and $L(\Sout,\cA(\Sin))$ as the evaluation of the model learned from past data on future data. 

Following \citet{Shapiro_1992}, we will refer to such quantitative comparisons as \textit{variational principles}. This name reflects the basic idea underlying the derivation of such results in the vein of optimality conditions in the classical calculus of variations: We consider small variations around $\cA(\Sin)$ and use the necessary conditions for optimality together with the properties of $f \mapsto L(S,f)$ to quantify the corresponding variations around $L(\Sin,\cA(\Sin))$. In mathematical statistics, similar ideas enter under the name of ``basic inequalities" \citep{VanDeGeer_2000,Paik_2025}.

\section{Variational principles for interpolation}\label{sec:interpolation}

In this section, we focus on interpolating algorithms $\cA$, i.e., the ones that achieve $L(S,\cA(S)) = 0$ for a given sample $S$. There is extensive literature on interpolation procedures (see, e.g., the text of \citet{Wendland_2005}) that aims to understand how accurately an interpolation procedure can characterize an unknown object based on finitely many measurements. The notion of measurement can be quite broad, but things can always be framed in such a way that the unknown object is a function $f : \cX \to \Reals$ on some ground set $\cX$ and the measurements are point evaluations of the form $f \mapsto f(x)$ for a fixed $x \in \cX$. Generalizing this, we can assume that the unknown $f$ is an element of a normed linear space $\cF$ and that the measurements are in a one-to-one correspondence with the elements of the dual space $\cF^*$ of continuous linear functionals on $\cF$. Thus, we are faced with the problem of (\textit{optimal}) \textit{recovery} of an unknown $f \in \cF$ from some set from a finite collection of linear measurements $S = \{(u_i,y_i) : i = 1, \dots, n\}$, where $u_i \in \cF^*$ and $y_i = \langle u_i, f \rangle$ \citep{Golomb_Weinberger_1958,Micchelli_1977}. We will adopt this framework to the setting of interpolating algorithms in machine learning.

Assume that the model class $\cF$ is itself a Banach space with Banach norm $\|\cdot\|_\cF$ and take $\cZ = \cF^* \times \Reals$. Thus, the elements of $\cZ$ are ordered pairs $(u,y)$, where $u$ is a continuous linear functional on $\cF$ and $y$ is a real number. For $S = \{ (u_i, y_i) : i = 1, \dots, n \} \subset \cZ$, we define the quadratic loss
\begin{align*}
L(S,f) := \frac{1}{n}\sum^n_{i=1} (y_i - \langle u_i,f \rangle )^2.
\end{align*}
Let us begin by analyzing algorithms that implement minimum norm interpolation: 
\begin{align}\label{eq:minimum-norm-interpolation}
\begin{split}
\text{minimize } & \|f\|_\cF \\
\text{subject to } & \langle u, f \rangle = y~~~\text{for all}~(u,y) \in S\,.
\end{split}
\end{align}
If a solution exists, we will denote it by $\hat{f}_S$. Existence and uniqueness of solutions can be established under reasonable regularity assumptions---for example, if $\cF$ is reflexive and the $u$'s entering into $S$ are linearly independent \citep{Wang_2021}. The out-of-sample properties of such minimum norm interpolators have been investigated extensively in the literature on optimal interpolation and optimal recovery \citep{Golomb_Weinberger_1958,Sard_1967,Sard_1973,Micchelli_1977,Traub_1988,Wendland_2005}. We briefly review and extend these results here.

Given a finite set $U = \{u_1,\dots,u_n\} \subset \cF^*$, define the linear \textit{sampling operator} $T_U : \cF \to \Reals^{n}$:
\begin{align*}
T_U f := \big(\langle u_1,f \rangle,\dots,\langle u_n, f\rangle \big)^\trn.
\end{align*}
For a sample (dataset) $S = \{(u_i,y_i) : i = 1,\dots,n\} \subset \cZ$, let $U(S) := \{u_1,\dots,u_n\}$. We will measure the dissimilarity between two datasets $\Sin$ and $\Sout$ in terms of how much variation the values of $f \mapsto \langle u, f \rangle$ exhibit on $u \in U(\Sout)$ when restricted to those $f$ that are constrained to interpolate $\Sin$. Specifically, we define
\begin{align}\label{eq:D_in_out}
D(\Sin,\Sout) := \frac{1}{\sqrt{|\Sout|}} \sup_{f \in \ker T_{U(\Sin)} \setminus \{0\}} \frac{ \| T_{U(\Sout)} f\|_2}{\|f\|_{\cF}},
\end{align}
where the norm in the numerator is the $\ell^2$ (Euclidean) norm on $\Reals^{|U(\Sout)|}$. The square of this quantity is equal to the value of the optimization problem
\begin{align}\label{eq:prob-D-in-out}
\begin{split}
\text{maximize } & \frac{1}{|\Sout|} \sum_{u \in U(\Sout)}  \langle u,f \rangle^2\\
\text{subject to } & \|f\|^2_\cF=1\\
& \langle v, f \rangle = 0 ~~~ \text{for all}~v \in U(\Sin)
\end{split}
\end{align}
With this definition, we can prove the following bound on the out-of-sample performance of the solution to Problem~\eqref{eq:minimum-norm-interpolation}:

\begin{theorem}\label{thm:GW59} Let two finite samples $\Sin$ and $\Sout$ be given, and let $S := \Sin \cup \Sout$. Assume that minimum-norm interpolating solutions $\hat{f}_S$ and $\hat{f}_{\Sin}$ exist. Then
\begin{align}
L(\Sout,\hat{f}_{\Sin}) &\le D^2(\Sin,\Sout) \|\hat{f}_{S} - \hat{f}_{\Sin}\|^2_\cF \label{eq:error_interp}
\end{align}
Moreover, if $(\cF,\|\cdot\|_\cF)$ is a Hilbert space, then 
\begin{align}
L(\Sout,\hat{f}_{\Sin}) &\leq D^2(\Sin,\Sout) (\|\hat{f}_S\|^2_\cF - \|\hat{f}_{\Sin}\|^2_\cF).\label{eq:error_interp-v2}
\end{align}
\end{theorem}
\begin{proof} Let $f$ be any element of $\cF$ that interpolates $S$. Then it also interpolates $\Sin$. Therefore $f - \hat{f}_{\Sin}$ belongs to the kernel of $T_{U(\Sin)}$. Applying the definition of $D(\Sin,\Sout)$, we have
\begin{align*}
L(\Sout,\hat{f}_{\Sin}) &= \frac{1}{|\Sout|} \|
T_{U(\Sout)}f - T_{U(\Sout)}\hat{f}_{\Sin} \|^2_2 \\
&= \frac{1}{|\Sout|} \|T_{U(\Sout)}  (f - \hat{f}_{\Sin}) \|^2_2 \\
&\le D^2(\Sin,\Sout) \| f - \hat{f}_{\Sin}\|^2_\cF.
\end{align*}
Taking $f = \hat{f}_S$, we obtain the first bound.

Now consider the case when $(\cF,\|\cdot\|_\cF)$ is a Hilbert space. Then, given any $f$ that interpolates $S$, $\hat{f}_{\Sin}$ solves the minimization problem
\begin{align*}
    \begin{split}
    \text{minimize } & \| f - g \|_\cF \\
    \text{subject to } & \langle u,g \rangle = y \text{ for all } (u,y) \in \Sin
    \end{split}
\end{align*}
Therefore, by the projection lemma, $f - \hat{f}_{\Sin}$ is orthogonal to $\hat{f}_{\Sin}$, which allows us to write
\begin{align*}
\| f - \hat{f}_{\Sin} \|^2_\cF &= \| f \|^2_\cF - 2 \langle f, \hat{f}_{\Sin} \rangle + \| \hat{f}_{\Sin}\|^2_\cF \\
&= \| f \|^2_\cF - \| \hat{f}_{\Sin}\|^2_\cF.
\end{align*}
Again, taking $f = \hat{f}_S$, we arrive at the claimed inequality.
\end{proof}

\noindent The above result has appeared in various works, starting with  \citet{Golomb_Weinberger_1958} and in two follow-up papers by \citet{Sard_1967,Sard_1973}. Our first inequality in Eq.~\eqref{eq:error_interp} bounds the out-of-sample error in terms of $D(\Sin,\Sout)$ and the minimum norm perturbation of $\hat{f}_{\Sin}$ required to interpolate $S$. The second inequality in Eq.~\eqref{eq:error_interp-v2} bounds the out-of-sample error in terms of the increase in norm when interpolating $S$ versus $\Sin$. In reproducing kernel Hilbert spaces,  \citet{LiangRecht23} showed that, when $\Sout$ is a singleton, Eq.~\eqref{eq:error_interp-v2} actually holds with \emph{equality}. When $|\Sout| > 1$, we can still show that \eqref{eq:error_interp-v2} is sharp:

\begin{theorem} Suppose $(\cF,\|\cdot\|_\cF)$ is a Hilbert space. Let $\Sin,\Sout,S$ be given as in Theorem~\ref{thm:GW59}, and assume that minimum-norm interpolating solutions $\hat{f}_S$ and $\hat{f}_{\Sin}$ exist. Then, for each $\eps > 0$ and each $r \ge \| \hat{f}_{\Sin}\|_\cF$, there exists some $\tilde{f} \in \cF$, such that $L(\Sin,\tilde{f}) = 0$, $\|\tilde{f}\|_\cF \le r$, and
\begin{align*}
L(\Sout,\tilde{f}) &\ge D^2(\Sin,\Sout) \| \tilde{f} -  \hat{f}_{\Sin} \|^2_\cF - \eps.
\end{align*}
\end{theorem}
\begin{proof} We follow the argument of \citet[Lemma~4]{Sard_1973}. Suppose first that $r = \| \hat{f}_{\Sin} \|_\cF$. Then we can take $\tilde{f} = \hat{f}_{\Sin}$ and, trivially,
\begin{align*}
L(\Sout,\tilde{f}) &\ge 0 = D^2(\Sin,\Sout) \| \tilde{f} - \hat{f}_{\Sin} \|^2_\cF \ge D^2(\Sin,\Sout) \| \tilde{f} - \hat{f}_{\Sin} \|^2_\cF - \eps.
\end{align*}
Suppose now that $r > \| \hat{f}_{\Sin}\|_\cF$. Let $c^2 := r^2 - \| \hat{f}_{\Sin}\|^2_\cF$. By the definition of $D(\Sin,\Sout)$, there exists some $g \in \ker T_{U(\Sin)}$ such that $\|g\|_\cF = 1$ and 
\begin{align*}
\frac{\| T_{U(\Sout)}g \|^2_2}{|\Sout|} \ge D^2(\Sin,\Sout) - \frac{\eps}{c^2}.
\end{align*}
If we take $\tilde{f} := \hat{f}_S + cg$, then $\tilde{f}$ evidently interpolates $\Sin$ and, since $\hat{f}_{\Sin}$ is a minimum-norm interpolator of $\Sin$, $\tilde{f} - \hat{f}_{\Sin}$ is orthogonal to $\hat{f}_{\Sin}$. Hence
\begin{align*}
r^2 - \| \hat{f}_{\Sin} \|^2_\cF &= c^2 = \| \tilde{f} - \hat{f}_{S}\|^2_\cF \ge \| \tilde{f} - \hat{f}_{\Sin} \|^2_\cF = \| \tilde{f} \|^2_\cF - \| \hat{f}_{\Sin} \|^2_\cF,
\end{align*}
which implies that $\|\tilde{f}\|_\cF \le r$. Moreover, 
\begin{align*}
L(\Sout,\tilde{f}) &= \frac{1}{|\Sout|} \| T_{U(\Sout)}\tilde{f} - T_{U(\Sout)}\hat{f}_S \|^2_2 \\
&= c^2 \frac{ \| T_{U(\Sout)}g \|^2_2}{|\Sout|} \\
&\ge c^2 \left(D^2(\Sin,\Sout) - \frac{\eps}{c^2}\right) \\
&= c^2D^2(\Sin,\Sout) - \eps \\
&= D^2(\Sin,\Sout) \left(\| \tilde{f} \|^2_\cF - \| \hat{f}_{\Sin} \|^2_\cF\right) - \eps\\
&= D^2(\Sin,\Sout) \|\tilde{f} - \hat{f}_{\Sin}\|^2_\cF - \eps.
\end{align*}
This completes the proof.
\end{proof}

\paragraph{Computation of $D(\Sin,\Sout)$.} In certain cases, the quantity $D(\Sin,\Sout)$ can be computed explicitly. For example, suppose that $\Sout$ is a singleton, i.e., $\Sout = \{(u_{n+1},y_{n+1})\}$ and $u_{n+1} \not\in U(\Sin) = \{u_1,\dots,u_n\}$. We have (see \citet[Sec.~6.6]{Luenberger_1969_opt})
\begin{align*}
\ker T_{U(\Sin)} = \left({\rm ran}\, T^*_{U(\Sin)}\right)^\perp,
\end{align*}
where $T^*_{U(\Sin)} : \Reals^n \to \cF^*$ is the adjoint of $T_{U(\Sin)}$ that maps any $\xi = (\xi_1,\dots,\xi_n)^\trn \in \Reals^n$ to
\begin{align*}
T^*_{U(\Sin)}\xi = \sum^n_{i=1} \xi_i u_i
\end{align*}
and where, for a linear subspace $L$ of $\cF^*$, $L^\perp := \{ f \in \cF: \langle u,f\rangle = 0, u \in L \}$. Moreover, 
\begin{align*}
{\rm ran}\, T^*_{U(\Sin)} = {\rm span}\, U(\Sin) = {\rm span} \{u_1,\dots,u_n\}.
\end{align*}
Assume $\cF$ is reflexive. Then standard duality arguments (see \citet[Sec.~5.8]{Luenberger_1969_opt}) lead to
\begin{align*}
\sup_{{f \in \ker T_{U(\Sin)}} \atop {\|f\|_\cF = 1}} \| T_{U(\Sout)}f \|_2 &= \sup_{f \in \ker T_{U(\Sin)} \atop \|f\|_\cF = 1} \langle u_{n+1}, f \rangle\\
&= \inf_{v \in {\rm ran}\, T^*_{U(\Sin)}} \| u_{n+1} - v \|_\cF
\end{align*}
which implies that
\begin{align*}
D(\Sin,\Sout) &= D(\{(u_i,y_i)\}^n_{i=1}, \{(u_{n+1},y_{n+1})\})  \\
&= {\rm dist}(u_{n+1}, {\rm span}\{u_1,\dots,u_{n}\}).
\end{align*}

As another example, consider the case when  $\cF$ is a reproducing kernel Hilbert space of functions on some set $\cX$ with embedding map $x \mapsto \varphi_x$ and kernel $K(x,y) = \langle \varphi_x, \varphi_y \rangle$. In this case, for each $x \in \cX$ the evaluation functional $f \mapsto f(x) = \langle \varphi_x, f \rangle$ is continuous, and we can equivalently consider data samples of the form $S = \{(x_i,y_i) : i = 1,\dots,n\} \subset \cX \times \Reals$.  Accordingly, we will use the notation $X(S) = \{x_1,\dots,x_n\} \subset \cX$ instead of $U(S) = \{\varphi_{x_1},\dots,\varphi_{x_n}\} \subset \cF$. In this setting, the value of~\eqref{eq:prob-D-in-out} can be computed by solving a generalized eigenvalue problem. To see this, let the samples
\begin{align*}
\Sin = \{(x_i,y_i) : i = 1,\dots,n \}, \qquad \Sout = \{ (\bar{x}_j, \bar{y}_j) : j = 1,\dots,m \}
\end{align*}
be given. Define the block matrices $K \in \Reals^{(n+m) \times (n+m)}$, $T_{{\mathsf{out}}} \in \Reals^{m \times (n+m)}$, and $T_{{\mathsf{in}}} \in \Reals^{n\times (n + m)}$ by
\begin{align*}
K &:= \begin{bmatrix}
K_{\ii} & K_{\io} \\
K_{\oi} & K_{\oo}
\end{bmatrix} \\
T_{{\mathsf{out}}} &:= \begin{bmatrix}
K_{\oi} & K_{\oo}
\end{bmatrix} \\
T_{{\mathsf{in}}} &:= \begin{bmatrix}
K_{\ii} & K_{\io}
\end{bmatrix}
\end{align*}
where $K_{\ii}$ has entries $K(x_i,x_j)$, $K_{\io}$ has entries $K(x_i,\bar{x}_j)$, etc. Then, by the representer theorem, the optimization problem~\eqref{eq:prob-D-in-out} is equivalent to the finite-dimensional problem
\begin{align}\label{eq:prob-D-in-out-rep}
\begin{split}
\text{maximize } & \frac{1}{m} \left\| T_{{\mathsf{out}}}\xi \right\|^2_2\\
\text{subject to } & \xi \in \Reals^{n+m} \\
& \xi^\trn K \xi = 1\\
& T_{{\mathsf{in}}}\xi = 0.
\end{split}
\end{align}
Assume that $K$ is nonsingular (this would be the case, for example, if the $x_i$'s and $\bar{x}_j$'s are all distinct and $K$ is a strictly positive definite kernel \citep{Sriperumbudur_2011}). Let $r := \dim \ker T_{{\mathsf{in}}}$ and let $U \Lambda U^\trn$, with $\Lambda = {\rm diag} (\lambda_1,\dots,\lambda_r)$, 
be the eigendecomposition of the $r \times r$ matrix $H := \Psi^\trn K^{-1/2} T^\trn_{{\mathsf{out}}} T_{{\mathsf{out}}}K^{-1/2} \Psi$, where $\Psi \in \Reals^{(n+m) \times r}$ is a matrix  whose columns form an orthonormal basis of $\ker T_{{\mathsf{in}}}$. (Note that the rank of $H$ cannot exceed $m$.)

\begin{theorem} The value of \eqref{eq:prob-D-in-out-rep} is equal to $\lambda_{\max}/m$, where $\lambda_{\max} = \max_{1 \le i \le r} \lambda_i$.
\end{theorem}
\begin{proof} Since the columns of $\Psi$ form an orthonormal basis of $\ker S$, $U$ is an $r \times r$ orthogonal matrix, and $K$ is nonsingular, it follows that $\ker T_{{\mathsf{in}}} = {\rm ran}\, K^{-1/2} \Psi U$. Hence,
\begin{align*}
\sup \left\{ \| T_{{\mathsf{out}}}\xi \|^2_2 : \xi^\trn K \xi = 1, \xi \in \ker T_{{\mathsf{in}}} \right\} &= \sup \left\{ \| T_{{\mathsf{out}}}K^{-1/2}\Psi U \eta \|^2_2: \| \Psi U \eta \|^2_2 = 1, \eta \in \Reals^r \right\} \\
&= \sup \left\{ \| T_{{\mathsf{out}}}K^{-1/2}\Psi U \eta \|^2_2: \|  \eta \|^2_2 = 1, \eta \in \Reals^r \right\}.
\end{align*}
Moreover,
\begin{align*}
\| T_{{\mathsf{out}}}K^{-1/2} \Psi U \eta \|^2_2 &= \eta^\trn U^\trn H U \eta  \\
&= \eta^\trn U^\trn U\Lambda U^\trn U \eta \\
&= \eta^\trn \Lambda \eta \\
&\le \lambda_{\max} \| \eta \|^2_2,
\end{align*}
and the supremum is achieved by letting $\eta$ be the normalized top eigenvector of $H$.
\end{proof}

Finally, note that we can crudely upper bound $D$ when the data is bounded. If $\|u\|_{\cF^*}\leq R$ for all $u \in \Sout$ then
\begin{align*}
D(\Sin,\Sout)^2 &= \frac{1}{|\Sout|} \sup_{f \in \ker T_{U(\Sin)} \setminus \{0\}} \frac{ \| T_{U(\Sout)} f\|_2^2}{\|f\|_{\cF}^2}\\
&=\frac{1}{|\Sout|} \sup_{f \in \ker T_{U(\Sin)} \setminus \{0\}} \frac{ \sum_{j=1}^{|\Sout|} \langle f, u_i \rangle^2}{\|f\|_{\cF}^2}\\
&\leq \frac{1}{|\Sout|}  \sum_{j=1}^{|\Sout|} \|u_i\|_{\cF^*}^2 \leq R^2\,.
\end{align*}

\paragraph{From deterministic to stochastic bounds.} If we make the additional strong assumption that the data sets $\Sin$ an $\Sout$ are collections of i.i.d.\ samples from some distribution $\mu$, these bounds recover standard generalization bounds in learning theory. For example, we have the following simple corollary. The subsampling of $m$ samples in this statement is a matter of convenience to keep the proof as short and simple as possible.

\begin{corollary}
Let $\mu$ denote a data-generating distribution. Suppose $\cF$ is a Hilbert space and there is an $f_\star \in \cF$ that globally interpolates all inputs. Also assume $\|u\|_\cF\leq R$ almost surely. Let $S_m$ denote a set of $m$ samples from $\mu$ and $\hat{f}_m$ denote the function interpolating these samples. Let $\Sout$ be an independent set of samples of arbitrary size.  Then for all $n$,
\begin{align*}
\min_{1\leq m \leq n} {\mathbf E}[ L(\Sout,\hat{f}_{m}) ] \leq \frac{ \|f_\star\|_{\cF}^2 R^2 }{n}\,.
\end{align*}
Here, the expectation is with respect to $\Sout$ and $\Sin$.
\end{corollary}

\begin{proof}
Fist, note that under the i.i.d.\ assumption, it suffices to consider the case where $\Sout$ is a singleton, because 
$$
{\mathbf E}\left[\frac{1}{|\Sout|}\sum^{|\Sout|}_{i=1} (y_i - \langle u_i,f \rangle )^2\right] = {\mathbf E}[(y - \langle u,f \rangle )^2]\,.
$$
Now, let $\hat{f}_n$ denote the random function that interpolates an i.i.d.\ sample of size $m$ from $\mu$. Theorem~\ref{thm:GW59} then implies
\begin{align*}
{\mathbf E}[L(\Sout,\hat{f}_{m})] &\leq R^2 ({\mathbf E}[\|\hat{f}_{m+1}\|^2_\cF] - {\mathbf E}[\|\hat{f}_{m}\|^2_\cF]).
\end{align*}
Summing this expression from $m=1$ to $m=n$ yields
\begin{align*}
\frac{1}{n} \sum_{m=1}^n {\mathbf E}[L(\Sout,\hat{f}_{m})] &\leq \frac{R^2 {\mathbf E}[\|\hat{f}_{n+1}\|^2_\cF]}{n}\leq \frac{R^2 \|f_\star\|^2_\cF}{n}
\end{align*}
which proves the claim.
\end{proof}

\section{Variational principles for hard-margin classifiers}\label{sec:duality}
We can derive similar deterministic generalization bounds for maximum-margin classification. We restrict our attention to Hilbert spaces and let $\hat{f}_S$ now denote a minimizer of the optimization problem
\begin{align}\label{eq:rkhs-max-margin}
\begin{array}{ll}
\text{minimize } & \frac{1}{2}\|f\|^2 \\
\text{subject to } & yf(x) \geq 1~~~\text{for all}~(x,y) \in S\,.
\end{array}
\end{align}
Here $y_i$ are $\pm 1$ labels. The solution of this optimization problem is often called a hard-margin support vector machine.

For this problem, we will use Lagrange multipliers to derive deterministic bounds on the generalization error. To do so, let's remind ourselves of the duality properties. Let $S = \{(u_i,y_i) : i = 1,\dots,n\} \subset \cZ$. The dual problem of Problem~\eqref{eq:rkhs-max-margin} is then
\begin{equation}\label{eq:rkhs-max-margin-dual}
	\begin{array}{ll}
	\text{maximize} & -\frac{1}{2} \left \| \sum_{i=1}^n \alpha_i y_i u_i \right\|^2 +  \sum_{i=1}^n \alpha_i\\
	\text{subject to} & \alpha \geq 0\,.
	\end{array}
\end{equation}
Let $f_S$ denote the optimal solution of~\eqref{eq:rkhs-max-margin}. Then we can write $f_S = \sum_{i=1}^n \alpha_i y_i u_i$, where $\alpha$ is the dual optimal solution. Moreover, by strong duality, we have the equality
\begin{align}\label{eq:norm-sum-eq}
\sum_{i=1}^n \alpha_i = \|f_S\|^2\,.
\end{align}
We can now use weak duality to derive the following:
\begin{lemma}
Let $S= \{(u_i,y_i) : i = 1,\dots,n\}$ be a data set, such that the maximum-margin solution of Problem~\eqref{eq:rkhs-max-margin} on $S$ exists.  Let $f_S$ denote the primal optimal solution and $\beta$ the dual optimal solution. Partition $S$ into disjoint subsets $\Sin$ and $\Sout$ and let $f_{\Sin}$ denote the primal optimal solution of the maximum margin problem on the set $\Sin$. Then for any $\gamma \in \mathbb{R}^n_+$, we have
\begin{align}\label{eq:mm-lagrange-bound}
\frac{1}{2} \Bigg\| \sum_{i \in \Sout} \beta_i y_i u_i \Bigg\|^2 
\geq  \frac{1}{2}\|f_S\|^2 -\frac{1}{2} \|f_{\Sin}\|^2 
\geq -\frac{1}{2} \Bigg\| \sum_{i \in \Sout} \gamma_i y_i u_i \Bigg\|^2+ \sum_{i \in \Sout} \gamma_i (1  -  y_i \langle f_{\Sin} , u_i \rangle )\,.
\end{align}
\end{lemma}

\begin{proof}
First, write $f_{\Sin} = \sum_{j \in \Sin} \eta_j y_j u_j$ where $\eta$ is the dual optimal solution of the max-margin problem on $\Sin$. Then we have for any $\gamma \geq 0$
\begin{align*}
	\frac{1}{2} \|f_S\|^2 & \geq \sum_{j \in \Sin} \eta_j + \sum_{i \in \Sout} \gamma_i -\frac{1}{2} \Bigg\| \sum_{j \in \Sin} \eta_j y_j u_j  + \sum_{i \in \Sout} \gamma_i y_i u_i \Bigg\|^2\\
	 	 &= \|f_{\Sin}\|^2 + \sum_{i \in \Sout} \gamma_i  -\frac{1}{2} \Bigg\| f_{\Sin}  + \sum_{i \in \Sout} \gamma_i y_i u_i \Bigg\|^2\\
	 	 &= \|f_{\Sin}\|^2 + \sum_{i \in \Sout} \gamma_i  - \frac{1}{2} \|f_{\Sin}\|^2  - \sum_{i\in \Sout} \gamma_i y_i \langle f_{\Sin} , u_i \rangle - \frac{1}{2}  \Bigg\|\sum_{i \in \Sout} \gamma_i y_i u_i \Bigg\|^2\\
		 &= \frac{1}{2} \|f_{\Sin}\|^2 + \sum_{i \in \Sout} \gamma_i (1  -  y_i \langle f_{\Sin} , u_i \rangle )- \frac{1}{2}  \Bigg\|\sum_{i \in \Sout} \gamma_i y_i u_i \Bigg\|^2
\end{align*}

Next, switching the roles of $f_S$ and $f_{\Sin}$ gives a similar chain,
\begin{align*}
	\frac{1}{2} \|f_{\Sin}\|^2 &\geq \sum_{j\in \Sin} \beta_j - \frac{1}{2} \Bigg\| \sum_{j \in \Sin} \beta_j  y_j u_j \Bigg\|^2\\
	&= \|f_{S}\|^2 - \sum_{i \in \Sout} \beta_i - \frac{1}{2} \Bigg\| f_S - \sum_{i \in \Sout} \beta_i y_i u_i\Bigg\|^2\\
	&= \frac{1}{2} \|f_{S}\|^2 - \sum_{i \in \Sout} \beta_i(1  -  y_i \langle f_S, u_i \rangle )
	- \frac{1}{2} \Bigg\|  \sum_{i \in \Sout} \beta_i y_i u_i\Bigg\|^2\\
	&\geq \frac{1}{2} \|f_{S}\|^2 - \frac{1}{2} \Bigg\|  \sum_{i \in \Sout} \beta_i y_i u_i\Bigg\|^2\,.
\end{align*}
The final inequality follows because the $\beta_i$ are nonnegative and $1  -  y_i \langle f_S, u_i \rangle \leq 0$ for all points in $S$.
Combining these two derivations proves the lemma.
\end{proof}
\noindent From this lemma, we immediately have the following bound analogous to Theorem~\ref{thm:GW59}:

\begin{theorem}\label{lemma:mm-batch-error}
Let two finite samples $\Sin$ and $\Sout$ be given, and let $S := \Sin \cup \Sout$. Assume that maximum-margin solution $f_S$ exists. Let $K_{\Sout}$ denote the kernel matrix formed from inner products of elements of $X(\Sout)$. Then for any $\hat{f}$ that correctly labels $S$,
\begin{align*}
\frac{1}{|\Sout|} \sum_{i\in \Sout} \max(1 - y_i f_{\Sin}(x_i),0)^2 &\le  \frac{\lambda_{\max}(K_{\Sout})}{|\Sout|}\Bigg( \|\hat{f} \|^2 - \|f_{\Sin}\|^2\Bigg). 
\end{align*}
\end{theorem}

\begin{proof}[Proof of Theorem~\ref{lemma:mm-batch-error}]
For any $\gamma$,
\begin{align*}
\Bigg\| \sum_{i \in \Sout} \gamma_i y_i u_i \Bigg\|^2 \leq \sum_{i \in \Sout} \gamma_i^2 \lambda_{\max}(K_{\Sout})
\end{align*}
Choosing 
\begin{align*}
\gamma_{i} =  \frac{2\max(1-y_i \hat{f}_{\Sin}(x_i),0)}{\lambda_{\max}(K_{\Sout})}
\end{align*}
and rearranging terms in second inequality of Equation~\eqref{eq:mm-lagrange-bound} proves the theorem.
\end{proof}
\noindent We can also derive a striking leave-one-out inequality for the support vector machine. 

\begin{theorem}\label{thm:svm-loo}
Let $S= \{(u_i,y_i) : i = 1,\dots,n\}$ be a data set where that maximum-margin solution of Problem~\eqref{eq:rkhs-max-margin} on $S$ exists.
Set $R:=\displaystyle\max_{1\leq i \leq n} \|u_i\|$. We then have the following deterministic bound on the leave-one-out error:
\begin{align}\label{eq:loo-bound}
	\frac{1}{n} \sum_{i=1}^n \max(1-y_i \langle f_{S\setminus i}, u_i \rangle,0) \leq\frac{R^2 \|f_S\|^2}{n}
\end{align}
\end{theorem}

\begin{proof}
For each $i$ and any $\gamma \geq 0$, the inequalities of Equation~\eqref{eq:mm-lagrange-bound} yield
\begin{align*}
\frac{1}{2} \beta_i^2 \| u_i\|^2  \geq
- \frac{1}{2} \gamma^2 \|x_i\|^2+  \gamma(1-y_i \langle f_{S\setminus i}, u_i \rangle)
\end{align*}
Choosing $\gamma = \max(1-y_i \langle f_{S\setminus i}, u_i \rangle,0)/\|u_i\|^2$ gives
\begin{align*}
	\max(1-y_i f_{S\setminus i}(x_i),0) \leq \|u_i\|^2 \beta_i\leq R^2 \beta_i
\end{align*}
Summing over both sides and using the identity~\eqref{eq:norm-sum-eq} will then prove the desired leave-one-out bound.
\end{proof}

This leave-one-out bound has an immediate translation to the probabilistic generalization model where the data points in $S$ are i.i.d.\ samples  from a distribution. Suppose that for any sample, $\|u\|\leq R$ almost surely. Also assume there exists some function $f$ of norm at most $B$ with $y f(x) \geq 0$ almost surely. Let $(x,y)$ be sampled i.i.d.\ from the same distribution. We have
\begin{align*}
	\Pr[ \operatorname{sign}(f_S(x)) \neq y ] = \Ex[ \mathbf{1}[ y f_S(x) \leq 0]
	&\leq\Ex[\max(1-y f_{S}(x),0) ] \,,
\end{align*}
and hence, Theorem~\ref{thm:svm-loo} immediately implies
\begin{align*}
	\Pr[ \operatorname{sign}(f_S(x)) \neq y ]\leq \frac{R^2B^2}{n}
\end{align*}
The right hand size is equal to the VC dimension of all functions of norm at most $B$ in Hilbert Space acting on data with norm at most $R$. Hence, by Theorem 1 of \citet{ehrenfeucht1989general}, this generalization error matches the best lower bound achievable for classification on this problem. That is, although Theorem~\ref{thm:svm-loo} makes no stochastic assumptions about the data generating process, it achieves the best possible generalization bound in expectation.

We note that Vapnik and Chervonenkis actually proved that Theorem~\ref{thm:svm-loo} was true in their 1974 book \emph{Theory of Pattern Recognition}~\citep{VC_OG_Book}. There, they started under the assumption that the data was sampled iid, but the argument works in a similar way to our analysis. One can first derive a leave-one-out bound for arbitrary data, and then compute the generalization error at the end by taking the expected value of the leave-one-out error. An English translation of Vapnik and Chervonenkis argument appears in Chapter 3 of \citet{hardt_recht_ppa_book}.

\section{A variational principle under quadratic growth assumption}\label{sec:quadratic}

In the preceding sections, we have focused on the ``classical" setting of linear functionals, minimum-norm interpolation, and max-margin classifiers. The theory in this setting relies heavily on the linear and/or convex structure of the problem under both the usual statistical assumptions and the deterministic framework of this paper. In this section, we develop variational principles in a setting that does not require global convexity. Instead, we assume that some form of curvature and/or strong convexity holds \textit{locally}. Our analysis makes use of ideas that go back to the work of  \citet{Shapiro_1992} (see also Ch.~4 of \citet{Bonnans_2000}). Recall that $\cM(S,\eps)$ denotes the set of all $\eps$-minimizers of $L(S,\cdot)$. We assume that the following \textit{quadratic growth} condition holds: There exists a positive constant $c > 0$ such that, for every $S$ the following inequality holds for all $f \in \cF \cap \cW(S)$, where $\cW(S)$ is some open, convex neighborhood of $\cM(S,0)$ in ${\mathfrak B}$:
\begin{align}\label{eq:quadratic_growth}
L(S,f) \ge \min_{f \in \cF} L(S,f) + c \cdot {\rm dist}(f,\cM(S,0))^2.
\end{align}
Let $\Sin$ and $\Sout$ be given, and define  $m(f) := L(\Sout,f) - L(\Sin,f)$. Operationally, $m(f)$ can be thought of as a generalization error (difference between out-of-sample and in-sample performance) of $f$. With this in mind, we introduce the quantity
\begin{align}\label{eq:kappa}
\kappa(\Sin,\Sout) := \sup \Bigg\{ \frac{|m(f') - m(f)|}{\|f'-f\|} : f \in \cM(\Sin,0), f' \in \cW(\Sin), f' \neq f \Bigg\},
\end{align}
which is a local Lipschitz constant that measures the variation of the generalization error of all models in a neighborhood of the minimizers of $f \mapsto L(\Sin,f)$. The theorem below then says that Lipschitz-continuous generalization error plus quadratic growth suffice to ensure that approximate minimizers of out-of-sample error are close to exact minimizers of in-sample error:

\begin{theorem}\label{thm:var_qg} Under the above assumptions, for any $\eps > 0$, any $f \in \cM(\Sout,\eps) \cap \cW(\Sin)$ satisfies
\begin{align*}
{\rm dist}(f,\cM(\Sin,0)) \le \frac{\kappa(\Sin,\Sout)}{c} + \sqrt{\frac{\eps}{c}}.
\end{align*}
\end{theorem}

\begin{proof} We follow the proof of Lemma~2.1 of \citet{Shapiro_1992}. Fix an arbitrary $f \in \cM(\Sout,\eps) \cap \cW(\Sin)$ and an arbitrary $f_0 \in \cM(\Sin,0)$. Then, using the definition of $m(\cdot)$ and $\eps$-optimality of $f$ for $L(\Sout,\cdot)$,
\begin{align*}
L(\Sin,f) - L(\Sin,f_0) &= L(\Sout,f) - L(\Sout,f_0) + m(f_0) - m(f) \\
&\le m(f_0) - m(f) + \eps.
\end{align*}
For an arbitrary $\gamma > 0$, we can find $f_0 \in \cM(\Sin,0)$, such that
\begin{align*}
\| f - f_0 \| \le {\rm dist}(f,\cM(\Sin,0)) + \gamma.
\end{align*}
Using this choice of $f_0$ together with the quadratic growth assumption \eqref{eq:quadratic_growth}, we can write
\begin{align*}
L(\Sin,f) - L(\Sin,f_0) &\ge c(\| f - f_0 \| - \gamma)^2.
\end{align*}
Let $\kappa := \kappa(\Sin,\Sout)$. It follows that $r := \| f - f_0 \|$ must satisfy the quadratic inequality
\begin{align*}
c(r - \gamma)^2 \le \kappa r + \eps
\end{align*}
which has the solution
\begin{align*}
r \le \frac{\kappa}{2c} + \gamma + \sqrt{\frac{\eps}{c} + \Bigg(\gamma + \frac{\kappa}{2c}\Bigg)^2 - \gamma^2}.
\end{align*}
Consequently,
\begin{align*}
{\rm dist}(f,\cM(\Sin,0)) &\le \lim_{\gamma \to 0} \Bigg( \frac{\kappa}{2c} + \gamma + \sqrt{\frac{\eps}{c} + \Bigg(\gamma + \frac{\kappa}{2c}\Bigg)^2 - \gamma^2} \Bigg) \\
&= \frac{\kappa}{2c} + \sqrt{\frac{\eps}{c} + \Bigg(\frac{\kappa}{2c}\Bigg)^2 } \\
&\le \frac{\kappa}{c} + \sqrt{\frac{\eps}{c}}.
\end{align*}
\end{proof}

The above result quantifies the closeness of the minimizers of $L(\Sin,\cdot)$ and $L(\Sout,f)$.  Under additional assumptions, we can obtain a generalization bound relating the evaluations $L(\Sout,f_{{\mathsf{in}}})$ and $L(\Sout,f_{{\mathsf{out}}})$. For simplicity, let us take $\cF = {\mathfrak B}$. We assume the following:
\begin{itemize}
    \item[(A1)] $f \mapsto L(S,f)$ is Fr\'echet-differentiable in $f$, with Lipschitz-continuous Fr\'echet derivative $DL(S,\cdot)$ -- there exists a constant $M > 0$ such that, for all $S$, 
    \begin{align*}
\| DL(S,f) - DL(S,f') \|_* \le M\|f-f'\|, \qquad  \text{for all } f,f' \in {\mathfrak B},
\end{align*}
where $\| \cdot \|_*$ denotes the dual norm in ${\mathfrak B}^*$;
    \item[(A2)] $f \mapsto DL(S,f)$ is locally metrically regular \citep{Dontchev_2009} around $\cM(S,0)$ -- there exists a constant $\alpha > 0$ such that, for each $S$, the inequality
\begin{align}\label{eq:Lojasiewicz}
\| DL(S,f) \|_* \ge \alpha\, {\rm dist}(f,\cM(S,0))
\end{align}
holds for all $f$ in an open, convex neighborhood $\cW_*(S)$ of $\cM(S,0)$.\end{itemize}
We first outline this argument in the statistical learning setting, following a suggestion of \citet{Rigollet_PC_2021}. Suppose that, for each $S = \{z_i\}^n_{i=1}$, the evaluation $L(S,f)$ is the sample loss of the form
\begin{align}\label{eq:average_ell}
L(S,f) = \frac{1}{n}\sum^n_{i=1} \ell(z_i,f),
\end{align}
where $f \mapsto \ell(z,f)$ is Fr\'echet-differentiable for each $z \in \cZ$. Assume that, for each $S$, $L(S,f)$ has a unique minimizer $f_S \in {\mathfrak B}$. Given a probability measure $\mu$ on $\cZ$, let
\begin{align*}
L_\mu(f) := {\mathbf E}[\ell(z,f)]
\end{align*}
denote the expected loss when $z \sim \mu$. Again, assume that $L_\mu(f)$ has a unique minimizer $f^*$.  Now let $\Sin$ and $\Sout$ be two independent tuples of i.i.d.\ draws from $\mu$, and let $f_{{\mathsf in}}$ be the (unique) minimizer of $L(\Sin,f)$. By Lipschitz continuity of $f \mapsto DL(S,f)$ and metric regularity of $f \mapsto DL(S,f)$, 
\begin{align*}
L(\Sout,f_{\mathsf{in}}) - L(\Sout,f^*) &\le \langle DL(\Sout,f^*), f_{\mathsf{in}} - f^* \rangle + \frac{M}{2} \| f_{\mathsf{in}} - f^* \|^2 \\
&\le  \langle DL(\Sout,f^*), f_{\mathsf{in}} - f^* \rangle + \frac{M}{2\alpha^2} \| DL(\Sin,f^*)  \|^2_* \\
&= \langle DL(\Sout,f^*), f_{\mathsf{in}} - f^* \rangle + \frac{M}{2\alpha^2} \| DL(\Sin,f^*) - DL_\mu(f^*) \|^2_*,
\end{align*}
where in the last line we have used the fact that $DL_P(f^*) = 0$. Now, taking expectations on both sides w.r.t.\ the randomness in $\Sout$ and using the fact that ${\mathbf E} DL(\Sout,f^*) = DL_\mu(f^*) = 0$, we obtain
\begin{align}\label{eq:MR_gen}
L_\mu(f_{\mathsf{in}}) - \min_{f} L_\mu(f) 
&\le \frac{M}{2\alpha^2} \| DL(\Sin,f^*) - DL_\mu(f^*) \|^2_*.
\end{align}
For $\Sin = \{z_i\}^n_{i=1}$,
\begin{align*}
    DL(\Sin,f^*) - DL_\mu(f^*) = \frac{1}{n}\sum^n_{i=1} \big( D\ell(z_i,f^*) - DL_\mu(f^*)\big)
\end{align*}
is an average of i.i.d.\ zero-mean random elements of ${\mathfrak B}^*$, and we can use concentration-of-measure arguments to obtain sharp error rates, both in expectation and with high probability.

Next, we adopt the logic of the above argument to the purely deterministic setting. However, while the above analysis assumed an infinite population and used the theoretical optimum (minimizer of the population risk) $f^*$ as a reference, here we need to frame things in terms of $f_{\mathsf{in}}$ and $f_{\mathsf{out}}$ only.

\begin{theorem} Under (A1) and (A2), for any $f_{\mathsf{out}} \in \cM(\Sout,0) \cap \cW_*(\Sin)$ there exists some $f_{\mathsf{in}} \in \cM(\Sin,0)$, such that
\begin{align*}
L(\Sout,f_{{\mathsf{in}}}) - \min_{f \in \cF} L(\Sout, f) \le \frac{M}{2\alpha^2} \| D L(\Sin, f_{{\mathsf{out}}}) - DL(\Sout, f_{{\mathsf{out}}}) \|^2_*.
\end{align*}
\end{theorem}
\begin{remark} The above theorem shows that there exists at least one minimizer of $L(\Sin,f)$ that performs well on $\Sout$. When $L(S,f)$ has a unique minimizer $f \in \cF$ for each $S$, we obtain a genuine generalization bound.
\end{remark}
\begin{proof} Let $f_{\mathsf{out}} \in \cM(\Sout,0) \cap \cW_*(\Sin)$ be given, and let $f_{\mathsf{in}}$ be the element of $\cM(\Sin,0)$ closest to $f_{\mathsf{out}}$:
\begin{align*}
\| f_{\mathsf{out}} - f_{\mathsf{in}} \| = {\rm dist}(f_{\mathsf{out}},\cM(\Sin,0)).
\end{align*}
Using the fact that $DL(\Sout,f_{{\mathsf{out}}}) = 0$ and Lipschitz continuity of the Fr\'echet derivative $DL(\Sout,\cdot)$, we have
\begin{align*}
L(\Sout,f_{{\mathsf{in}}}) - L(\Sout,f_{{\mathsf{out}}}) &\le \langle DL(\Sout,f_{\mathsf{out}}),f_{\mathsf{in}}-f_{\mathsf{out}}\rangle + \frac{M}{2} \|f_{{\mathsf{out}}} - f_{{\mathsf{in}}}\|^2 \\
&= \frac{M}{2}{\rm dist}(f_{\mathsf{out}},\cM(\Sin,0))^2.
\end{align*}
By metric regularity, we can upper-bound ${\rm dist}(f_{\mathsf{out}},\cM(\Sin,0))^2$ as follows:
\begin{align*}
{\rm dist}(f_{\mathsf{out}},\cM(\Sin,0))^2 &\le \frac{1}{\alpha^2} \| DL(\Sin, f_{{\mathsf{out}}})\|^2_* \\
&= \frac{1}{\alpha^2} \| DL(\Sin,f_{{\mathsf{out}}}) - DL(\Sout,f_{{\mathsf{out}}})\|^2_*.
\end{align*}
Putting these estimates together, we obtain the claimed inequality.
\end{proof}
\noindent When $L(S,f)$ has the form in \eqref{eq:average_ell}, the above theorem gives a purely deterministic analogue of the probabilistic bound \eqref{eq:MR_gen}.

\section{A variational principle for convex evaluations}\label{sec:evaluations}

We will now assume that, for each fixed $S$, the functional $f \mapsto L(S,f)$ is convex and that the set $\cM(S,0)$ is a singleton. Let $\Sin$ and $\Sout$ be given. Let $f_{{\mathsf{in}}}$ and $f_{{\mathsf{out}}}$ denote the corresponding unique minimizers of $L(\Sin,\cdot)$ and $L(\Sout,\cdot)$ and define the following \textit{localized} analogues of the growth condition \eqref{eq:quadratic_growth} and of the Lipschitz constant $\kappa(\Sin,\Sout)$ defined in \eqref{eq:kappa}:
\begin{align*}
&h(\Sin, \delta) := \inf \left\{ L(\Sin,f) - \min_{f \in \cF} L(\Sin,f) : \| f - f_{{\mathsf{in}}} \| = \delta \right\}, \\
&K(\Sin, \Sout, \delta) := \sup \left\{ |m(f) - m(f_{{\mathsf{in}}})|: \| f - f_{{\mathsf{in}}} \| \le \delta \right\},
\end{align*}
where $m(f) = L(\Sout,f) - L(\Sin,f)$ is defined in the preceding section. Then we have the following result:

\begin{theorem}\label{thm:var_cvx} Under the above assumptions,
\begin{align*}
\| f_{{\mathsf{in}}} - f_{{\mathsf{out}}} \| \le \inf \left\{ \delta > 0: K(\Sin,\Sout,\delta) < h(\Sin, \delta) \right\}.
\end{align*}
\end{theorem}
\begin{proof} The result and the proof can be extracted from the proof of Lemma~2 of \citet{Hjort_1993}. Fix any $\delta > 0$, such that
\begin{align}\label{eq:delta_assumption}
K(\Sin,\Sout,\delta) < h(\Sin,\delta)
\end{align}
and suppose that $\| f_{{\mathsf{out}}} - f_{{\mathsf{in}}} \| > \delta$. Then we can write
\begin{align*}
f_{{\mathsf{out}}} = f_{{\mathsf{in}}} + r g
\end{align*}
for some $r > \delta$ and some $g$ with $\| g \| = 1$. By convexity,
\begin{align*}
L(\Sout,f_{{\mathsf{in}}} + \delta g) &= L(\Sout,(\delta/r)f_{{\mathsf{out}}} + (1-\delta/r) f_{{\mathsf{in}}}) \\
&\le (\delta/r) L(\Sout,f_{{\mathsf{out}}}) + (1-\delta/r) L(\Sout, f_{{\mathsf{in}}}).
\end{align*}
Rearranging and using definitions of $h$ and $K$, we obtain
\begin{align*}
&(\delta/r) \cdot [L(\Sout,f_{{\mathsf{out}}}) - L(\Sout,f_{{\mathsf{in}}})] \nonumber\\
&\qquad\ge L(\Sout,f_{{\mathsf{in}}} + \delta g) - L(\Sout, f_{{\mathsf{in}}}) \\
&\qquad= [L(\Sin,f_{{\mathsf{in}}} + \delta g) - L(\Sin,f_{{\mathsf{in}}})]  + m(f_{{\mathsf{in}}}+\delta g) - m(f_{{\mathsf{in}}}) \\
&\qquad\ge h(\Sin,\delta) - K(\Sin,\Sout,\delta).
\end{align*}
By the choice of $\delta$, the right-hand side is strictly positive, which contradicts the optimality of $f_{{\mathsf{out}}}$ for $L(\Sout,\cdot)$. Consequently, 
\begin{align*}
\|f_{{\mathsf{out}}} - f_{{\mathsf{in}}} \| \le \delta.
\end{align*}
Taking the smallest $\delta$ satisfying \eqref{eq:delta_assumption},  we obtain the inequality in the theorem.
\end{proof}

We can use Theorem~\ref{thm:var_cvx} to obtain the same estimate as in Theorem~\ref{thm:var_qg} (for $\eps = 0$) under the quadratic growth assumption \eqref{eq:quadratic_growth}. If $L(S,f)$ satisfies the quadratic growth condition, then
\begin{align*}
h(\Sin,\delta) \ge c \delta^2.
\end{align*}
Moreover, from the definition of $\kappa(\Sin,\Sout)$ it follows that
\begin{align*}
 K(\Sin,\Sout,\delta) \le \kappa(\Sin,\Sout)\delta.
\end{align*}
Consequently,
\begin{align*}
\inf \left\{ \delta > 0: K(\Sin,\Sout,\delta) < h(\Sin, \delta) \right\} \le \frac{\kappa(\Sin,\Sout)}{c}.
\end{align*}
In statistical learning literature, $\Sout$ is replaced by its infinite population proxy, and in that case the critical radius $\delta$ that achieves the best trade-off between $K$ and $h$ determines the rates of convergence of the excess risk as $|\Sin| \to \infty$. This is, essentially, the idea behind localization arguments popularized by \citet{Koltchinskii_2006}.

\section*{Acknowledgments}

The authors would like to thank Juan Carlos Perdomo, Philippe Rigollet, and Andrej Risteski for their careful reading of this paper and for their constructive suggestions. M.R.~would also like to acknowledge the organizers of the meeting on \textit{Statistical Thinking in the Age of AI: Decision-Making and Reliability} that was held in December 2025 at the Centre International de Rencontres Math\'ematiques (CIRM) in Luminy, France, for the invitation to present an early version of this work, as well as the participants of the meeting for a robust discussion.

\bibliography{geogen.bbl}

\end{document}